\documentclass[12pt]{article}
\usepackage{enumitem}
\usepackage{soul} 
\usepackage[colorlinks]{hyperref}            
\usepackage{color}
\usepackage{graphicx,subfigure,amsmath,amssymb,amsfonts,bm,epsfig,epsf,url,dsfont}
\usepackage{times}
\usepackage{bbm}      
\usepackage{booktabs}
\usepackage{cases}
\usepackage{fullpage}
\usepackage[small,bf]{caption}
\usepackage{algorithm}
\usepackage{algpseudocode}
\usepackage{natbib}
\usepackage[top=1in,bottom=1in,left=1in,right=1in]{geometry}

\renewcommand{\phi}{\varphi}

\renewcommand{\P}{\mathbb{P}}
\newcommand{\E}{\mathbb{E}}

\newcommand{\R}{\mathbb{R}}

\newcommand{\cK}{\mathcal{K}}

\newcommand{\cN}{\mathcal{N}}

\newcommand{\cL}{\mathcal{L}}

\def\ds1{\mathds{1}}
\renewcommand{\epsilon}{\varepsilon}

\newcommand{\argmin}{\mathop{\mathrm{argmin}}}

\renewcommand{\tilde}{\widetilde}

\newlength{\minipagewidth}
\newcommand{\beq}{\begin{equation}}
\newcommand{\eeq}{\end{equation}}

\newcommand{\beqa}{\begin{eqnarray}}
\newcommand{\eeqa}{\end{eqnarray}}

\newcommand{\beqan}{\begin{eqnarray*}}
\newcommand{\eeqan}{\end{eqnarray*}}

\def\ba#1\ea{\begin{align*}#1\end{align*}} 
\def\banum#1\eanum{\begin{align}#1\end{align}} 

\newtheorem{theorem}{Theorem}

\newtheorem{lemma}{Lemma}

\newcommand{\BlackBox}{\rule{1.5ex}{1.5ex}}  
\newenvironment{proof}{\par\noindent{\bf Proof\ }}{\hfill\BlackBox\\[2mm]}

\begin{document}

\title{Sparsity, variance and curvature in multi-armed bandits}

\author{S\'ebastien Bubeck \\
Microsoft Research
\and Michael B. Cohen \thanks{This work was done while M. B. Cohen and Y. Li were at Microsoft Research.} \\
MIT
\and 
Yuanzhi Li \footnotemark[1]\\
Princeton University}
\date{\today}

\maketitle

\abstract{In (online) learning theory the concepts of sparsity, variance and curvature are well-understood and are routinely used to obtain refined regret and generalization bounds. In this paper we further our understanding of these concepts in the more challenging limited feedback scenario. We consider the adversarial multi-armed bandit and linear bandit settings and solve several open problems pertaining to the existence of algorithms with favorable regret bounds under the following assumptions: (i) sparsity of the individual losses, (ii) small variation of the loss sequence, and (iii) curvature of the action set. Specifically we show that (i) for $s$-sparse losses one can obtain $\tilde{O}(\sqrt{s T})$-regret (solving an open problem by Kwon and Perchet), (ii) for loss sequences with variation bounded by $Q$ one can obtain $\tilde{O}(\sqrt{Q})$-regret (solving an open problem by Kale and Hazan), and (iii) for linear bandit on an $\ell_p^n$ ball one can obtain $\tilde{O}(\sqrt{n T})$-regret for $p \in [1,2]$ and one has $\tilde{\Omega}(n \sqrt{T})$-regret for $p>2$ (solving an open problem by Bubeck, Cesa-Bianchi and Kakade). A key new insight to obtain these results is to use regularizers satisfying more refined conditions than general self-concordance.}

\section{Introduction}
In this paper we resolve several open problems in multi-armed bandit theory. Let us first recall the general setting of bandit linear optimization on a compact set $\cK \subset \R^n$ (the classical multi-armed bandit problem corresponds to $\cK=\{e_1,\hdots,e_n\}$, the canonical basis in $\R^n$). It can be described as the following sequential game: at each time step $t=1, \hdots, T$, a player selects an action $a_t \in \cK$, and simultaneously an adversary selects a linear loss function $\ell_t : \cK \rightarrow [-1,1]$. The player's feedback is its suffered loss, $\ell_t(a_t)$. Equivalently we will view the loss function $\ell_t$ as a vector in the polar body $\cK^{\circ} := \{h : \forall x \in \cK, |h \cdot x| \leq 1\}$, and thus we write $\ell_t(x) = \ell_t \cdot x$. The player has access to external randomness, and can select her action $a_t$ based on the history $H_t=(a_s, \ell_s(a_s))_{s<t}$. The player's perfomance at the end of the game is measured through the 
{\em pseudo-regret} (the expectation is with respect to the randomness in her strategy) :
\begin{equation} \label{eq:regret}
R_T = \E \sum_{t=1}^T \ell_t(a_t) - \min_{x \in \cK}\E \sum_{t=1}^T \ell_t(x) ,
\end{equation}
which compares her cumulative loss to the smallest cumulative loss she could have obtained had she known the sequence of loss functions. We refer to \cite{BC12} for the history of this problem, and we simply mention that the minimax rate for the regret is known to be $\tilde{\Theta}(n \sqrt{T})$ without further assumptions on $\cK$, and for the special case where $\cK=\{e_1,\hdots,e_n\}$ (i.e., the multi-armed bandit problem) it is $\Theta(\sqrt{n T})$. 

We consider three basic open problems in bandit theory (description below), each one part of a more general trend in learning theory/online learning, namely (i) exploiting sparsity, (ii) faster learning for ``easy data", and (iii) interplay between curvature and learning\footnote{Note that the terms sparsity and curvature in the paper's title apply respectively to the losses and the action set. They could also apply respectively to the action set and to the losses, see e.g. \cite{LLZ09} and \cite{HL14}. We do not consider these (very different) settings here.}. In fact these problems are possibly the easiest at the intersection of bandit theory and topics (i), (ii), (iii). Thus, given the flurry of activity on these topics and on bandit theory in recent years, we believe that they epitomize the difficulty of adapting full information tools to limited feedback scenarios. In particular we hope that the tools we develop to resolve these problems will find broader applicability.
\newline

\noindent
\textbf{Sparse multi-armed bandit, \cite{KP16}.} Consider the multi-armed bandit problem with the additional assumption that at each time step $t \in [T]$ the loss vector $\ell_t \in [-1,1]^n$ only has $s$ non-zero entries. Trivially the best regret one can hope for in this setting is $\Omega(\sqrt{s T})$. Kwon and Perchet ask whether there is a strategy with regret matching this lower bound (possibly up to logarithmic factors). Surprisingly the state of the art for this problem is the standard $O(\sqrt{n T})$ bound, or in other words prior to this present work it was not known whether sparsity of the losses can be exploited in a bandit setting\footnote{We note however that for {\em non-negative} losses (which should intuitively be a much easier case than say sparse {\em non-positive} losses, a.k.a. sparse gains), Kwon and Perchet already answered positively the question, see Section \ref{sec:obstacles}.}.
\newline

\noindent
\textbf{Small variation bound for multi-armed bandit, \cite{HK09}.} Consider again the multi-armed bandit problem with the additional assumption that the loss sequence $(\ell_1, \hdots, \ell_T) \in ([-1,1]^n)^T$ has a {\em small variation} $Q:= \sum_{t=1}^T \|\ell_t - \frac1{T} \sum_{s=1}^T \ell_s\|_2^2$ (note that $Q \leq n T$). The COLT 2011 open problem by Hazan and Kale ask whether there exists a strategy with regret $\tilde{O}(\sqrt{Q})$ (\cite{HK11}). The current state of the art remains \cite{HK09} which gives a strategy with regret $\tilde{O}(n^2 \sqrt{Q})$. We also note that \cite{GL16} showed that for any fixed $Q > \log(T)$ one cannot obtain a regret smaller than $\Omega(\sqrt{Q})$ for all sequences with variation $Q$.
\newline

\noindent
\textbf{Linear bandit on $\ell_p^n$ balls, \cite{BCK12}.} Consider the linear bandit problem on $\cK= \{x \in \R^n : \|x\|_p \leq 1\}$. The general minimax rate show that for any $p \geq 1$ there exists a strategy with regret $\tilde{O}(n \sqrt{T})$, and furthermore this is optimal for $p=\infty$. It is easy to see that for $p=1$ the problem can be reduced to the classical multi-armed bandit (in dimension $2n$) and thus there exists a strategy with regret $\tilde{O}(\sqrt{n T})$. In \cite{BCK12} it is shown that the latter regret can also be achieved for $p=2$. No other result is known for this problem, and a natural conjecture\footnote{This conjecture was mentioned in talks related to \cite{BCK12}.} would be that $\tilde{O}(\sqrt{n T})$ is achievable for any $p \in [1,2]$, and that the minimax regret then degrades ``smoothly'' for $p>2$ until $\tilde{\Omega}(n \sqrt{T})$ for $p=\infty$.
\newline

We resolve all the above problems, constructing strategies with respective regret bounds $\tilde{O}(\sqrt{s T})$, $\tilde{O}(\sqrt{Q})$, and $\tilde{O}(\sqrt{n T})$ for $p \in [1,2]$. Furthermore we show that in fact for $p>2$ the minimax regret (for large $T$) is $\tilde{\Theta}(n \sqrt{T})$. We also introduce the following more constrained version of bandit linear optimization, which we call {\em starved bandit}. In this model the player only observes feedback if she plays $a_t$ from a {\em fixed} distribution $\mu \in \Delta(\cK)$, where $\mu$ is chosen by the player at the beginning of the game. Thus the player is ``information starved''. One can motivate such a setting in various ways, think for instance of applications where logging information on users is discouraged for privacy reasons. It is easy to see that one must have regret $\Omega(T^{2/3})$ for the starved multi-armed bandit game, and that the same lower bound also applies to starved linear bandit on $\ell_p^n$ unit ball with $p=1$. Perhaps surprisingly we show that $\sqrt{T}$-type regret is achievable for the starved bandit for any $p \in (1,2]$ and {\em not} achievable for any $p>2$.
\newline

A key feature of our work that enables these improved regret bounds is that we avoid resorting to ``global'' smoothness of the regularizers. Slightly more precisely, as we will recall shortly, an important step in the analysis of FTRL (Follow The Regularized Leader) is to show that the regularizer is well-conditioned. Since the groundbreaking work \cite{AHR08} it has been realized that self-concordance (\cite{NN94}) exactly gives such a good conditioning {\em for all directions}. In this paper we use more refined properties of the regularizers, by noticing that one only needs the well-conditioning in directions (and magnitudes) {\em attainable with loss estimators}.
\newline

Next we describe more formally our main results.

\subsection{Main results}
The brief algorithms' description given in the theorem statements below use standard bandit theory terminology which is recalled in Section \ref{sec:reminders}. Note also that in this paper we assume that the parameters of the game (such as the time horizon $T$, or the variation of the loss sequence) are known. Standard methodology (such as the doubling trick, or more sophisticated variants of it) can be used to circumvent this issue.
\newline

We start with a theorem resolving the sparse bandit open problem by Kwon and Perchet (notice that if $\|\ell_t\|_0 \leq s$ and $\|\ell_t\|_{\infty}\leq1$ then $\sum_{t=1}^T \|\ell_t\|_2^2 \leq s T$).
\begin{theorem} \label{th:sparse}
There exists a multi-armed bandit strategy such that for any loss sequence satisfying $\sum_{t=1}^T \|\ell_t\|_2^2 \leq L$ (and $\ell_t \in [-1,1]^n$) one has 
$$R_T \leq 10 \sqrt{L \log(n)} + 20 n \log(T) ~.$$ 

In fact this can be achieved with the FTRL strategy (with standard unbiased loss estimator) with the regularizer $\Phi(x) = \sum_{i=1}^n x(i) \log x(i) - \gamma \sum_{i=1}^n \log x(i)$, learning rate $\eta = \min\left(\frac1{5} \sqrt{\frac{\log(T)}{L}}, \frac{1}{15 n}\right)$, and soft-exploration parameter $\gamma = 2 \eta$. 
\end{theorem}

The difficulty in achieving a result such as Theorem \ref{th:sparse} is that standard multi-armed bandit algorithms {\em explore too much}. In fact as was noted in \cite{HK11} for the variation bound open problem (the same observation holds for the sparse bound open problem): ``We note that EXP3 itself has $\Omega(\sqrt{T})$ regret, since it mixes with the uniform distribution every iteration to enable sufficient exploration. Hence, the desired algorithm should be a little different from EXP3, incorporating just enough exploration proportional to the variation in the data.'' Our new idea to achieve this is to introduce {\em soft exploration}, by adding to the regularizer a little bit of the log-barrier for the positive orthant. This new hybrid regularizer and its analysis is one of our key contribution. We give detailed intuition for it in Section \ref{sec:intuition}. It also allows to solve the variation bound open problem:

\begin{theorem} \label{th:variance}
There exists a multi-armed bandit strategy and a numerical constant $C>0$ such that for any loss sequence satisfying $\sum_{t=1}^T \|\ell_t - \frac1{T} \sum_{s=1}^T \ell_s\|_2^2 \leq Q$ (and $\ell_t \in [-1,1]^n$) one has 
$$R_T \leq C \sqrt{Q \log(n)} + C n \log^2(T) ~.$$

In fact this can be achieved by combining the Hazan-Kale reservoir sampling idea with the strategy of Theorem \ref{th:sparse}
\end{theorem}

Next we give our main theorems for linear bandit on $\ell_p^n$ balls. Notice that the polar of the $\ell_p^n$ ball is the $\ell_q^n$ ball with $q=p/(p-1)$.

\begin{theorem} \label{th:UBellp}
Let $p \in (1,2]$. There exists a linear bandit algorithm playing on the unit ball of $\ell_p^n$ such that
$$R_T \leq 2^{\frac{6}{p-1}} \sqrt{n T \log(T)} ~.$$
\end{theorem}

Our lower bound construction for $\ell_p^n$ balls with $p>2$ uses Gaussian losses which satisfy the constraint $\|\ell_t\|_q^q \leq 1$ only in expectation. Note that from standard Gaussian concentration the same bound (up to a logarithmic factor) then holds with high probability. We work with Gaussian losses mostly for clarity of exposition, and at the expense of technical complications one could use losses which satisfy the bound $\|\ell_t\|_q^q \leq 1$ almost surely. We also note that the lower bound is only valid in the large $T$ regime, which is necessary since there exist intermediate regimes of $(T,n)$ where a better regret than $n \sqrt{T}$ is achievable.
 
\begin{theorem} \label{th:LBellp}
Let $p >2$ and $T \geq n^{\max\left(2, \frac{p-1}{p-2}\right)}$. There exists a numerical constant $C>0$ such that for any linear bandit algorithm playing on the unit ball of $\ell_p^n$, there exists $(\ell_t)_{t \in [T]}$, i.i.d. Gaussian random variables in $\R^n$ such that
\begin{equation} \label{eq:gaussianbound}
\E \|\ell_t\|_q^q \leq 1 ~,
\end{equation}
and 
$$\E R_T \geq C n \sqrt{T} ~.$$
\end{theorem}

We recall the starved bandit setting introduced above. At the beginning of the game the player chooses an exploration distribution $\mu \in \Delta(\cK)$. At any time $t$ the player can choose to play $a_t$ at random, either from $\mu$ or from an adaptive distribution $p_t$ (where $p_t$ depends on the observed feedback so far). The loss of the player is $\ell_t(a_t)$. The feedback is either (i) nothing if $a_t$ was played from $p_t$, or (ii) the standard bandit feedback $\ell_t(a_t)$ if $a_t$ was played from $\mu$. For sake of simplicity we assume that if $\cK$ contains the (signed) canonical basis then $\mu$ is uniform on the (signed) canonical basis. 

We observe that Theorem \ref{th:UBellp} holds true for the starved linear bandit framework too (indeed the strategy we give to prove Theorem \ref{th:UBellp} is a starved bandit strategy). Our main additional result for this setting is to show that for any $p$ not covered by Theorem \ref{th:UBellp} one cannot achieve $\sqrt{T}$-type regret:

\begin{theorem} \label{th:starved}
For any strategy for the starved multi-armed bandit there exists a loss sequence such that $R_T \geq \frac1{20} n^{1/3} T^{2/3}$. The same lower bound holds for the starved linear bandit on the $\ell_1^n$ ball. Furthemore for any $p>2$ there exists a constant $C>0$ such that for any starved linear bandit algorithm playing on the unit ball of $\ell_p^n$, there exists $(\ell_t)_{t \in [T]}$, i.i.d. Gaussian random variables in $\R^n$ satisfying \eqref{eq:gaussianbound} and such that
$$\E R_T \geq C n^{\frac{q}{2 + q}} T^{\frac{2}{2 + q}} ~.$$
\end{theorem}

\subsection{Notation}
We use the following (standard) notation: $\Delta(\cK)$ for the set of probability measures supported on $\cK$, $\Delta = \{x \in \R_+^n : \sum_{i=1}^n x(i) =1\}$ for the simplex, $\|x\|_p = \left( \sum_{i=1}^n |x(i)|^p\right)^{1/p}$ for the $\ell_p^n$ norm, $\Phi^*(\theta) = \sup_{x \in \R^n} \theta \cdot x - \Phi(x)$ for the Fenchel dual of $\Phi : \R^n \rightarrow \overline{\R}$, $D_{\Phi}(x,y) = \Phi(x) - \Phi(y) - \nabla \Phi(y) \cdot (y-x)$ for the Bregman divergence associated to $\Phi$, $\|h\|_x = \sqrt{\nabla^2 \Phi(x)[h,h]}$ for the local norm induced by $\Phi$ at $x$, $\|h\|_{x,*} = \sqrt{(\nabla^2 \Phi(x))^{-1}[h,h]}$ for the dual local norm, $\odot$ for the Hadamard product (i.e., entrywise product of vectors), and $\succeq$ for the positive semi-definite ordering on matrices. 

\section{Bandit theory reminders} \label{sec:reminders}
We give a few brief reminders of multi-armed bandit and linear bandit theory.

\subsection{Full information strategies}
In this section we assume that $\cK$ is a convex body in $\R^n$. We fix a learning rate $\eta >0$ and a mirror map $\Phi : \R^n \rightarrow \overline{\R}$, that is a strictly convex and differentiable map with $\nabla \Phi(\R^n) = \R^n$ and diverging gradient as one approaches the boundary of its domain. The following theorem is a standard result on the mirror descent strategy for online linear optimization (with full information), see e.g., [Theorem 5.5, \cite{BC12}].

\begin{theorem} \label{th:MD}
Let $\ell_1, \hdots, \ell_T \in \R^n$ be a fixed sequence of loss vectors and let $x_1, \hdots, x_T \in \cK$ be defined by: $x_1 = \argmin_{x \in \cK} \Phi(x)$ and
\begin{equation} \label{eq:careful}
x_{t+1} = \argmin_{x \in \cK} D_{\Phi}(x, \nabla \Phi^*(\nabla \Phi(x_t) - \eta \ell_t)) .
\end{equation}
Then one has for any $x \in \cK$,
\begin{equation} \label{eq:basicMD}
\sum_{t=1}^T \ell_t \cdot (x_t - x) \leq \frac{\Phi(x) - \Phi(x_1)}{\eta} + \frac{1}{\eta} \sum_{t=1}^T D_{\Phi^*}\bigg(\nabla \Phi(x_t) - \eta \ell_t, \nabla \Phi(x_t)\bigg)~.
\end{equation}
Futhermore assuming that the following implication holds true for any $y_t \in \R^n$,
\begin{equation} \label{eq:condMD}
\nabla \Phi(y_t) \in [\nabla\Phi(x_t), \nabla\Phi(x_{t}) - \eta \ell_t] \Rightarrow \nabla^2 \Phi(y_t) \succeq c \nabla^2 \Phi(x_t)
\end{equation}
one obtains
\begin{equation} \label{eq:basicregret}
\sum_{t=1}^T \ell_t \cdot (x_t - x) \leq \frac{\Phi(x) - \Phi(x_1)}{\eta} + \frac{\eta}{2 c} \sum_{t=1}^T \|\ell_t\|_{x_t,*}^2 ~.
\end{equation}
\end{theorem}

We will also use the lazy variant of mirror descent, also known as FTRL (Follow The Regularized Leader), and its corresponding ``primal only'' analysis. In particular while for mirror descent one has to check that $\Phi$ is ``well-conditioned'' on a ``dual segment'' (equation \eqref{eq:condMD}) we will see below that for FTRL one needs to check the well-conditioning on a ``primal segment'' (equation \eqref{eq:condFTRL}). Note also that mirror descent and FTRL give the same update equation when $\Phi$ is a barrier for $\cK$ (see e.g., \cite{Bub15}), which is often the case in bandit scenario.

\begin{theorem} \label{th:lazyMD}
Let $\ell_1, \hdots, \ell_T \in \R^n$ be a fixed sequence of loss vectors and let $x_1, \hdots, x_T \in \cK$ be defined by:
\begin{equation} \label{eq:careful2}
x_{t} = \argmin_{x \in \cK} \eta \sum_{s=1}^{t-1} \ell_s \cdot x + \Phi(x)  .
\end{equation}
Then one has for any $x \in \cK$,
\begin{equation} \label{eq:induction}
\sum_{t=1}^T \ell_t \cdot (x_t - x) \leq \frac{\Phi(x) - \Phi(x_1)}{\eta} + \sum_{t=1}^T \ell_t \cdot (x_t - x_{t+1}) ~.
\end{equation}
Futhermore assuming that the following implication holds true for any $y_t \in \R^n$,
\begin{equation} \label{eq:condFTRL}
y_t \in [x_t, x_{t+1}] \Rightarrow \nabla^2 \Phi(y_t) \succeq c \nabla^2 \Phi(x_t)
\end{equation}
then one has that \eqref{eq:basicregret} holds true with the term $\frac{\eta}{2c}$ replaced by $\frac{2 \eta}{c}$.
\end{theorem}

\begin{proof}
The proof of \eqref{eq:induction} is a classical one-line induction (sometimes referred to as the Be-The-Leader lemma). We turn to \eqref{eq:basicregret} and note that it suffices to show that $\|x_t - x_{t+1}\|_{x_t} \leq \frac{2 \eta}{c} \|\ell_t\|_{x_t,*}$. Observe that, using a Taylor expansion, for some $y_t \in [x_t,x_{t+1}]$ one has, with the notation $\Phi_t(x):= \eta \sum_{s=1}^t \ell_s \cdot x + \Phi(x)$ (thus $x_{t+1} \in \argmin \Phi_t$ and $x_t \in \argmin \Phi_t - \eta \ell_t$),
\begin{eqnarray*}
\frac12 \|x_t-x_{t+1}\|_{y_t}^2 = \Phi_t(x_t) - \Phi_t(x_{t+1}) - \nabla \Phi_t(x_{t+1}) \cdot (x_t - x_{t+1}) & \leq & \Phi_t(x_t) - \Phi_t(x_{t+1}) \\
& \leq & \eta {\ell}_t \cdot (x_t - x_{t+1}) ~.
\end{eqnarray*}
Using that $\nabla^2 \Phi(y_t) \succeq c \nabla^2 \Phi(x_t)$ one also has $\|x_t - x_{t+1}\|_{x_t}^2 \leq \frac{1}{c} \|x_t-x_{t+1}\|_{y_t}^2$ and thus
$$\|x_t-x_{t+1}\|_{x_t}^2 \leq \frac{2 \eta}{c} {\ell}_t \cdot (x_t - x_{t+1}) \leq \frac{2 \eta}{c} \|\ell_t\|_{x_t,*} \|x_t -x_{t+1}\|_{x_t} ~,$$
which concludes the proof. 
\end{proof}

\subsection{Bandit strategies} \label{sec:classicalbandit}
In addition to choosing a regularizer, a bandit strategy also rely on a sampling scheme, that is a map $p : \mathrm{conv}(\cK) \rightarrow \Delta(\cK)$ such that $\E_{X \sim p(x)} X = x$. One then runs FTRL (or mirror descent), with the (unobserved) true losses $\ell_t$ replaced by estimators $\tilde{\ell}_t$ (constructed based on the observed feedback). Moreover instead of playing the point $x_t$ recommended by FTRL, i.e., $x_t = \argmin_{x \in \mathrm{conv}(\cK)} \sum_{s=1}^{t-1} \tilde{\ell}_s \cdot x + \Phi(x)$, one plays at random $a_t \sim p(x_t)$ (where the sampling is done independently of the past given $x_t$). The key point is that if the loss estimator is unbiased, i.e., $\E_{a_t \sim p(x_t)} \tilde{\ell}_t = \ell_t$, then one has for any $x \in \cK$,
$$\E \sum_{t=1}^T \ell_t \cdot (a_t - x) = \E \sum_{t=1}^T \tilde{\ell}_t \cdot (x_t - x) ~,$$
and thus one can use Theorem \ref{th:MD} or Theorem \ref{th:lazyMD} to bound the regret. In particular assuming that one can prove the well-conditioning condition \eqref{eq:condMD} or \eqref{eq:condFTRL}, the key quantity to control is the ``variance'' of the loss estimator appearing in \eqref{eq:basicregret}, namely $\E \ \|\tilde{\ell}_t\|_{x_t,*}^2$.
\newline

To illustrate the above discussion let us briefly recall the classical multi-armed bandit setting (i.e., $\cK=\{e_1, \hdots, e_n\}$) with \textbf{nonnegative losses}. We use mirror descent with $\Phi(x) = \sum_{i=1}^n x(i) \log x(i)$,
the sampling scheme $p : \Delta \rightarrow \Delta(e_1,\hdots e_n)$ is simply the identity map (in the sense that $\P_{a \sim p(x)}(a=e_i) = x(i)$), and the unbiased loss estimator is 
$$\tilde{\ell}_t(i) = \frac{\ell_t(i)}{x_t(i)} \ds1\{a_t = e_i\} ~.$$

The key is to observe that since $\tilde{\ell}_t$ has nonegative entries, one has that \eqref{eq:condMD} is satisfied with $c=1$, and thus \eqref{eq:basicregret} gives
$$R_T \leq \frac{\log(n)}{\eta} + \frac{\eta}{2} \sum_{t \in [T], i \in [n]} \E \ \|\tilde{\ell}_{t}\|_{x_t,*}^2 ~.$$
The last thing to observe is that, since $\|h\|_x^2 = \sum_{i=1}^n \frac{h(i)^2}{x(i)}$, one has
$$\E \ \|\tilde{\ell}_{t}\|_{x_t,*}^2 = \E \sum_{i=1}^n x_t(i) \tilde{\ell}_t(i)^2 = \E \sum_{i=1}^n x_t(i) \frac{{\ell}_t(i)^2}{x_t(i)} \ds1\{a_t=e_i\} = \|\ell_t\|_2^2 ~.$$
Thus with an appropriate choice of $\eta$ one gets
\begin{equation} \label{eq:classicalregret}
R_T \leq \sqrt{\frac{\log(n)}{2} \sum_{t=1}^T \|\ell_t\|_2^2} ~.
\end{equation}
As a side note we observe that using the polynomial INF regularizer of \cite{AB09} (see Section \ref{sec:intuition} for a brief reminder on the INF regularizer), for any primal dual pair $p, q \geq 1$, one obtains an algorithm with a regret bound scaling in $\frac{q}{q-1} \sqrt{n^{1/q} \sum_{t=1}^T \|\ell_t\|_{2p}^2}$. 

\section{Sparsity and variation bounds for multi-armed bandit}
We start first by describing some basic obstacles to obtain a sparsity type bound in Section \ref{sec:obstacles}. Then in Section \ref{sec:intuition} we give some intuition for our new ``hybrid regularizer'', $\sum_{i=1}^n x(i) \log(x(i)) - \gamma \sum_{i=1}^n \log(x(i))$, that is the weighted combination of the negentropy and the logarithmic barrier for the positive orthant\footnote{The logarithmic barrier was recently used as a regularizer for bandits in \cite{Foster16} to obtain first order regret bounds. We note however that the behavior of our hybrid regularizer is fundamentally different from using only the log-barrier term.}. The extra logarithmic barrier term can be understood as a soft way to encourage exploration (to the contrary of the usual forced exploration). Finally in Section \ref{sec:proofsparse} we prove Theorem \ref{th:sparse} (this section is self-contained and does not require reading the two previous subsections).

\subsection{Basic obstacles} \label{sec:obstacles}
The basic issue is that \eqref{eq:classicalregret} only holds for nonnegative losses\footnote{Notice that one cannot simply shift the losses as this could potentially suppress sparsity.}. The reason nonnegativity was needed is that the well-conditioned assumption for the negentropy $\Phi$, equation \eqref{eq:condMD}, crucially relies on the fact that (note that $\nabla \Phi = \log, \nabla^2 \Phi = \mathrm{diag}(1/x)$) for $\log(y) = \log(x) - \ell$ with $\ell \geq 0$ one has $1/y\geq 1/x$. A standard fix to maintain the latter inequality approximately true for general losses is to ensure that the magnitude of the (estimated) loss is controlled. Indeed \eqref{eq:condMD} is satisfied for some constant $c$ provided that almost surely $\|\eta \tilde{\ell}_t\|_{\infty} \leq \log(1/c)$. This almost sure control can be achieved by adding forced exploration, as was done in the original adversarial multi-armed bandit paper \cite{ACFS03}, that is the sampling scheme is now $(1-n \gamma) x_t + \gamma \ds1$, or in words explore uniformly at random with probability $n \gamma$ and otherwise play from $x_t$. Indeed in this case $\|\eta \tilde{\ell}_t\|_{\infty} \leq {\eta} / {\gamma}$,
and thus the well-conditioned assumption \eqref{eq:condMD} is satisfied when $\gamma \simeq \eta$. However the added regret (with respect to $i^* \in [n]$) suffered by the extra exploration is exactly $\gamma \sum_{i,t} (\ell_t(i) - \ell_t(i^*))$. This latter term destroys the scaling with sparsity (for example if $\ell_t = -e_{i^*}$ then this term is of order $\gamma (n-1) T \simeq \eta n T$). More prosaically, the uniform exploration might make us miss out on a $n \gamma$ fraction of the ``gains'' of the best arm, which could be far too much. We also observe that the recently proposed implicit exploration by \cite{KNVM14} (see also \cite{Neu15}) suffers from the exact same issue.
\newline

We also note that, without going into any technical details, the case of arbitrary losses seem harder than the case of nonnegative losses. Indeed the former contains the case of nonpositive losses, or equivalently nonnegative {\em gains}. Sparse nonnegative losses mean that most arms are performing well and only a handful are to be avoided. On the other hand sparse nonnegative gains mean that most arms are bad, and only a handful are performing well. Intuitively, finding this small set of good arms hiding in a sea of bad arms is harder than avoiding a small set of bad arms in a sea of good arms.

\subsection{Intuition for the hybrid regularizer} \label{sec:intuition}
The intuition is divided in two parts: (i) the fact that the added regret for $\gamma >0$ is controlled, and (ii) that the well-conditioning still holds. 

For the first part we start with a slightly different point of view on extra (forced) exploration. It is easy to check that adding extra exploration exactly corresponds to taking the regularizer to be a ``negatively shifted negentropy'': $\sum_{i=1}^n (x(i) - \gamma) \log(x(i) - \gamma)$. For such a regularizer the range $\Phi(x) - \Phi(x_1)$ is controlled only for $x$'s such that $\min_{i \in [n]} x(i) > \gamma$. In the worst case the gap between the regret with respect to such $x$'s, and with respect to an arbitrary $x$ can be as large as $n \gamma T$, and since the well-conditioned assumption requires $\gamma \simeq \eta$ this leads us to the extra term $\eta n T$. On the other hand for the hybrid barrier one can compare to $x$'s with $\min_{i \in [n]} x(i) =1/\mathrm{poly}(T)$, only at the expense of a term of the form $\frac{\gamma n \log(T)}{\eta}$. Thus provided that the well-conditioning assumption remains true for $\gamma \simeq \eta$ (this is the key part to verify) the hybrid regularizer could lead to a bound of the form \eqref{eq:classicalregret} up to to an extra additive term of order $n \log(T)$.

For the well-conditioning intuition we first recall the INF parametrization of a regularizer (\cite{ABL14}): For $\psi : \mathbb{R} \rightarrow \R$, let $\Phi$ be defined by $\nabla \Phi^*(x) := (\psi(x_i))_{i \in [n]}$. The negentropy regularizer exactly corresponds to $\psi(s) = \exp(s)$ while adding forced extra exploration with probability $n \gamma$ can be achieved by taking $\psi(s) = \exp(s) + \gamma$. The hybrid regularizer essentially corresponds to taking $\psi(s)$ to be the exponential function when $\psi(s) \geq \gamma$, and otherwise to be roughly like $\frac{\gamma \log \gamma}{s}$. In particular we see that the well-conditioning is satisfied for $\gamma \simeq \eta$ when the played arm has probability greater than $\gamma$ (since in this case everything behaves essentially as with forced exploration), and on the other hand when the played arm has probability smaller $\gamma$, its probability $x$ is of the form $1/L$ and the updated probability is $1/(L + 1/x) \simeq x$, and thus the well-conditioning also holds in this case.

\subsection{Proof of Theorem \ref{th:sparse}} \label{sec:proofsparse}

Observe that the hybrid regularizer $\Phi$ is lower bounded by the negentropy in the sense that $\nabla^2 \Phi(x) \succeq \mathrm{diag}(1/x(i))$. Thus the standard argument of Section \ref{sec:classicalbandit} shows that 
$$\E \ \|\tilde{\ell}_t\|_{x_t, *}^2 \leq \|\ell_t\|_2^2 ~.$$
In particular, using Theorem \ref{th:lazyMD}, it only remains to check \eqref{eq:condFTRL}.
The next lemma is the key justification for our new regularizer.

\begin{lemma} \label{lem:critical}
Let $\Phi$ be the hybrid regularizer, $\eta>0$, $L \in \R^n$, $\xi \in \R$, $L':=L + \xi e_1$,
$$x := \argmin_{y \in \Delta} \eta L \cdot y + \Phi(y) \; \text{and} \; x' := \argmin_{y \in \Delta} \eta L' \cdot y + \Phi(y) ~.$$
Assuming that $|\xi| \leq C/x(1)$ for some $C>0$ and that $\gamma \geq \eta C$, one has for any $i \in [n]$, and any $u \in (0,1)$,
$$\max\left(\frac{x'(i)}{x(i)}, \frac{x(i)}{x'(i)} \right) \leq \max\left(\exp\left(\frac{1}{\frac{\gamma}{\eta C}-1}\right), \frac1{1-\gamma-u} \exp(\gamma n / u) \right) ~.$$
\end{lemma}

For example with $C=1$, $u=1/2$, $\gamma=2 \eta$, and $\eta \leq \frac1{15 n}$ one obtains
$$\max\left(\frac{x'(i)}{x(i)}, \frac{x(i)}{x'(i)} \right) \leq 3 ,$$
which means in particular (notice that $\nabla^2 \Phi(x) = \mathrm{diag}(1/x(i) + \gamma / x(i)^2)$) that for any $y_t \in [x_t, x_{t+1}]$ one has
$$\nabla^2 \Phi(x_t) \preceq 9 \nabla^2 \Phi(y_t) ~,$$
which finishes the proof of Theorem \ref{th:sparse} up to straightforward calculations.

\begin{proof}
First note that the KKT conditions for $x$ and $x'$ show that there exist $\lambda, \lambda' \in \R$ such that
\begin{equation} \label{eq:KKT}
\eta L + \nabla \Phi(x) = \lambda \ds1 , \; \eta L' + \nabla \Phi(x') = \lambda' \ds1 ~.
\end{equation}
Also note that $\nabla^2 \Phi(x)$ is diagonal with positive entries.
\newline

\noindent
\textbf{Step 1:} We show that $\lambda'$ and $x'(i)$ for $i \neq 1$ are increasing with $\xi$, while $x'(1)$ is decreasing with $\xi$. By differentiating \eqref{eq:KKT} one gets
\begin{equation} \label{eq:diffKKT}
\left. \frac{d \lambda'}{d \xi} \right. \ds1 = \eta e_1 + \nabla^2 \Phi(x) \left. \frac{d x'}{d \xi} \right. ~.
\end{equation}
By multiplying the above equation with $(\nabla^2 \Phi(x))^{-1}$ and summing over the coordinates (recall that $\sum_{i=1}^n \left. \frac{d x'(i)}{d \xi} \right. =0$) one obtains
$\left. \frac{d \lambda'}{d \xi} \right. > 0$.
In particular using this in \eqref{eq:diffKKT} one obtains for any $i \neq 1$,
$\frac{d x'(i)}{d\xi} > 0$,
and thus
$ \frac{d x'(1)}{d\xi} < 0$.
\newline

\noindent
\textbf{Step 2:} We now show that the first coordinate has a small multiplicative change. Substracting the two identities in \eqref{eq:KKT} one obtains, since $\nabla \Phi(x) = (1 + \log x(i) - \gamma / x(i))_{i \in [n]}$,
\begin{equation} \label{eq:KKT1}
\lambda' - \lambda + \log \frac{x(1)}{x'(1)} + \gamma \left( \frac{1}{x'(1)} - \frac{1}{x(1)}\right) = \eta \xi ~.
\end{equation}
Observe that that by Step 1 all the terms on the lhs have the same sign and thus
\begin{equation} \label{eq:XX}
|\lambda' - \lambda| + \left|\log \frac{x(1)}{x'(1)}\right| + \gamma \left| \frac{1}{x'(1)} - \frac{1}{x(1)}\right| = \eta | \xi | ~.
\end{equation}
In particular we have
$$\left| \frac{1}{x'(1)} - \frac{1}{x(1)}\right| \leq \frac{\eta C / \gamma}{x(1)} \Leftrightarrow \frac{x(1)}{x'(1)} \in [1-\eta C /\gamma, 1+\eta C / \gamma] ~.$$
Also note that that for any $s \in (0,1)$, $\max\left(1+s, \frac1{1-s} \right) \leq \exp\left(\frac1{\frac1{s}-1}\right)$.
\newline

\noindent
\textbf{Step 3:} Assuming that $x(1) \geq \gamma - \eta C$ we show that all the other coordinates also have a small multiplicative change (the case $x(1) < \gamma - \eta C$ is dealt with in the next step). Substracting the two identities in \eqref{eq:KKT} one obtains for any $i \neq 1$,
\begin{equation} \label{eq:KKTi}
\log \frac{x(i)}{x'(i)} + \gamma \left( \frac{1}{x'(i)} - \frac{1}{x(i)}\right) = \lambda - \lambda' ~.
\end{equation}
In particular since the two terms on the left hand side in \eqref{eq:KKTi} have the same sign one has
\begin{equation} \label{eq:diffL}
\left| \log \frac{x(i)}{x'(i)} \right| + \gamma  \left|\frac{1}{x'(i)} - \frac{1}{x(i)}\right| = |\lambda - \lambda'| ~.
\end{equation}
Next we also observe that thanks to \eqref{eq:XX}:
$$| \lambda - \lambda'| \leq \eta |\xi| \leq \frac{\eta C}{x(1)} ~.$$
In particular together with \eqref{eq:diffL} we proved that if $x(1) \geq \gamma - \eta C$ then one has
$$\left| \log \frac{x(i)}{x'(i)} \right| \leq \frac{1}{\frac{\gamma}{\eta C} - 1} ~.$$ 
\newline

\noindent
\textbf{Step 4:} Finally we show that if $x(1) \leq  \gamma - \eta C$ one also has that all the other coordinates have a small multiplicative change. Let $I:=\{i \neq 1 \; \text{s.t.} \; \min(x(i), x'(i)) \geq u/n \}$ (notice that, by Step 1, the minimum is attained uniformly either at $x$ or $x'$). Then thanks to \eqref{eq:diffL} one has for any $i \in I$,
$$\left| \log \frac{x(i)}{x'(i)} \right| \geq |\lambda - \lambda'| - \gamma n/u ~,$$
and thus
$$1 \geq \sum_{i \in I} \min(x(i), x'(i)) \exp(|\lambda - \lambda'| - \gamma n/u) ~.$$
Observe that if $\min(x(i),x'(i)) = x(i)$ for some $i \in I$ then one has 
$$\sum_{i \in I} \min(x(i), x'(i)) = \sum_{i \in I} x(i) \geq 1 - (\gamma - \eta C) - u ~,$$
while if $\min(x(i),x'(i)) = x'(i)$ for some $i \in I$ then one has (thanks to Step 2)
$$\sum_{i \in I} \min(x(i), x'(i)) = \sum_{i \in I} x'(i) \geq 1 - \frac{\gamma - \eta C}{1 - \frac{\eta C}{\gamma}} - u = 1-\gamma - u ~.$$
Thus we have 
$$1 \geq (1-\gamma - u) \exp(|\lambda - \lambda'| - \gamma n / u) ~,$$
which concludes the proof (recall that by \eqref{eq:diffL} one has for any $i \neq 1$, $\left| \log \frac{x(i)}{x'(i)} \right| \leq |\lambda - \lambda'|$).
\end{proof}

\subsection{Variation bound for multi-armed bandit}
We only give a brief sketch of proof of Theorem \ref{th:variance}, as it is essentially a straightforward combination of the proof of Theorem \ref{th:sparse} together with the arguments of \cite{HK09}. In particular we ignore explicit numerical constants with the notation $O$.

First note that it is easy to see from \eqref{eq:induction} that the following bound holds for full information FTRL under the well-conditioning assumption \eqref{eq:condFTRL}: for any sequence $m_1,\hdots, m_{T} \in \R^n$ and with $m_{T+1}=0$ one has
\begin{equation} \label{eq:HKregret}
\sum_{t=1}^T \ell_t \cdot (x_t - x) \leq \frac{\Phi(x) - \Phi(x_1)}{\eta} + \frac{2 \eta}{c} \sum_{t=1}^T 
\|\ell_t - m_t\|_{x_t,*}^2 +  \sum_{t=1}^{T+1} \|m_t - m_{t-1}\|_2 ~.
\end{equation}
The strategy of Hazan and Kale is to use a small portion of ``exploration'' rounds to estimate $\mu_t = \frac1{t} \sum_{s=1}^t \ell_s$ by some $\tilde{\mu}_t$ and then use it to center the loss estimator (for the non-``exploration'' rounds) by setting for any $i \in [n]$:
$$\tilde{\ell}_t(i) = \frac{(\ell_t - \tilde{\mu}_t)(i)}{x_t(i)} \ds1\{a_t = e_i\} + \tilde{\mu}_t(i) ~.$$
More precisely by doing an exploration round with probability $k n / t$ at round $t$ (the so-called ``reservoir sampling'', here $k>0$ is a parameter of the algorithm) one can obtain an estimator $\tilde{\mu}_t$ such that $\E \ \tilde{\mu}_t = \mu_t$ and $\mathrm{Var}(\tilde{\mu}_t) \leq \frac{Q}{k t}$. Moreover the added regret from those rounds is $O(k n \log(T))$. Thus using the bound \eqref{eq:HKregret} with $m_t = {\mu}_t$ it only remains to bound the terms $\eta \sum_{t=1}^T 
\|\tilde{\ell}_t - {\mu}_t\|_{x_t,*}^2$ and $\sum_{t=1}^{T+1} \|{\mu}_t - {\mu}_{t-1}\|_2$. The latter term is easily controlled by $O(\sqrt{n} \log(Q))$, see Lemma 12 in \cite{HK09}. On the other hand for the former term one gets
$$\E \ \|\tilde{\ell}_t - \mu_t\|_{x_t,*}^2  \leq 2 \E \ \|\tilde{\ell}_t - \tilde{\mu}_t\|_{x_t,*}^2 + 2 \E \ \|\tilde{\mu}_t - \mu_t\|_{x_t,*}^2 = 2 \E \|\ell_t - \mu_t\|_2^2 + 2 \mathrm{Var}(\tilde{\mu}_t) ~,$$
and thus $\eta \E \sum_{t=1}^T 
\|\tilde{\ell}_t - {\mu}_t\|_{x_t,*}^2 = O(\eta Q (1+\log(T)/k))$, which easily concludes the proof up to straigthforward computations.

\section{Regular and starved linear bandits on $\ell_p^n$ balls}

In this section we prove the results related to linear bandits on $\ell_p^n$ balls. Recall that $q=p/(p-1)$.

\subsection{Proof of Theorem \ref{th:UBellp}}
Let $p \in (1,2]$. We first describe a new strategy to play on $\ell_p^n$ balls based on a non-self-concordant barrier (when $p \neq 2$).
Let $d(x)=1-\|x\|_p^p$, and $\Phi(x) = - \log d(x)$ (notice that for $p\neq2$ the Hessian of $\Phi$ blows up at $0$, and thus $\Phi$ cannot be self-concordant).
We play FTRL with regularizer $\Phi$ and with sampling scheme given by: with probability $\max(d(x), \gamma)$ play uniformly in $\{e_1, -e_1, \hdots, e_n, -e_n\}$, and otherwise play $x/\|x\|_p$. Note that this not unbiased, but rather ``$\gamma$-biased'', which adds a $\gamma T$ term to the regret. The estimator is defined by $\tilde{\ell}_t = n \frac{\ell_t \cdot \tilde{x}_t}{1-\|x_t\|_p, \gamma)} \tilde{x}_t$ if played uniformly in $\{e_1, -e_1, \hdots, e_n, -e_n\}$, and $\tilde{\ell}_t = 0$ otherwise. 
\newline

While $\Phi$ is not self-concordant, the next lemma shows that one still has some form of well-conditioning (though not \eqref{eq:condMD}) that will turn out to be sufficient to control the regret.

\begin{lemma} \label{lem:locnormellp}
Let $x, \ell \in \R^n$ such that $\|x\|_p < 1$, $\|\ell\|_0 =1$ and $\|\ell\|_2 \leq 1$. Let $y \in \R^n$ such that $\nabla \Phi(y) \in [\nabla \Phi(x), \nabla \Phi(x) + \ell]$. Then one has for $p\in [1,2]$,
$$\|\ell\|_{y,*}^2 \leq \frac{2^{\frac{3}{p-1}} d(x)}{p(p-1)} \sum_{i=1}^n (|x(i)|^{2-p} + |\ell(i)|^{\frac{2-p}{p-1}}) \ell(i)^2 ~.$$
\end{lemma}

Before moving to the proof of Lemma \ref{lem:locnormellp} we show how to use it to control the variance of the loss estimator. The proof of Theorem \ref{th:UBellp} is then straightforward from \eqref{eq:basicMD} and Lemma \ref{lem:localnormbound}.

\begin{lemma} \label{lem:localnormbound}
The above strategy satisfies for any $y_t \in \R^n$ such that $\nabla \Phi(y_t) \in [\nabla \Phi(x_t), \nabla \Phi(x) - \eta \tilde{\ell}_t]$
$$\E_{a_t} \|\tilde{\ell}_t\|_{y_t,*}^2 \leq \frac{2^{\frac{4}{p-1}}}{p-1} n ~.$$
\end{lemma}

\begin{proof}
Note that $\|\eta \tilde{\ell}_t\|_2 \leq n \eta/\gamma$. Thus by Lemma \ref{lem:locnormellp} we have, provided that $\gamma \geq n \eta$,
$$\|\tilde{\ell}_t\|_{y_t,*}^2 \leq \frac{2^{\frac{3}{p-1}} d(x_t)}{p(p-1)} \E \sum_{i=1}^n (|x_t(i)|^{2-p} + |\eta \tilde{\ell}_t(i)|^{\frac{2-p}{p-1}}) \tilde{\ell}_t(i)^2 ~.$$
We now bound separately the two terms. For the first one we have (note that $1-\|x\|_p \geq \frac{1}{p} (1-\|x\|_p^p)$ and thus $d(x_t) \leq p \max(1-\|x_t\|_p, \gamma)$)
$$d(x_t) \E_{a_t} \sum_{i=1}^n |x_t(i)|^{2-p} \tilde{\ell}_t(i)^2 \leq p n \sum_{i=1}^n |x_t(i)|^{2-p} \ell_t(i)^2  \leq p n ~,$$
where the second inequality follows from Holder's inequality with $\frac{2}{q}+ \frac{2-p}{p}=1$. Now we bound the second term (note that $\frac{2-p}{p-1}+2=q$)
$$d(x_t) \E_{a_t} \sum_{i=1}^n  |\eta \tilde{\ell}_t(i)|^{\frac{2-p}{p-1}} \tilde{\ell}_t(i)^2 \leq p n \sum_{i=1}^n  |\ell_t(i) \eta n / \gamma|^{\frac{2-p}{p-1}} {\ell}_t(i)^2 \leq p n \sum_{i=1}^n \ell_t(i)^q \leq p n ~,$$
which concludes the proof.
\end{proof}

We give now a few preliminary results before proving Lemma \ref{lem:locnormellp}.

\begin{lemma} \label{lem:hessellp}
One has for any $x \in \R^n$ such that $\|x\|_p < 1$,
$$\nabla^2 \Phi^*(\nabla \Phi(x)) \preceq \frac{d(x)}{p(p-1)} \textsf{diag}(|x|^{2-p}) ~.$$
\end{lemma}

\begin{proof}
Straightforward derivations show that
\begin{equation} \label{eq:gradellp}
\nabla \Phi(x) = \frac{ p \cdot \textsf{sign}(x) \odot |x|^{p - 1}}{1 - \| x\|_p^p} ~,
\end{equation}
\begin{align*}
\nabla^2 \Phi(x) &= \frac{p(p - 1) \textsf{diag}(|x|^{p - 2})}{1 - \| x \|_p^p} + \frac{p^2 \left(\textsf{sign}(x) \odot |x|^{p - 1} \right)^{\otimes 2} }{ (1 - \| x\|_p^p)^2}
\\
& \succeq  \frac{p(p - 1) \textsf{diag}(|x|^{p - 2})}{1 - \| x \|_p^p} ~,
\end{align*}
which directly implies the lemma.
\end{proof}

\begin{lemma} \label{lem:ellpcritical}
Let $v \in \R^n$ and $\ell \in \R^n$ such that $\|\ell\|_0 =1$ and $\|\ell\|_2 \leq 1$. Denote $x=\nabla \Phi^*(v)$ and $y=\nabla \Phi^*(v+\ell)$. Then one has
\begin{align}
& d(y) \leq 4 d(x) ~, \label{eq:stayclosetobdy} \\
& |y(i)|\leq 2^{\frac{3}{p-1}} |x(i)| + |2 \ell(i)|^{\frac{1}{p-1}} \label{eq:coordwisebd} ~.
\end{align}
\end{lemma}

\begin{proof}
Observe that by definition (recall \eqref{eq:gradellp}) one has
$$|x(i)| = \left( \frac{ |v(i)| d(x)}{p} \right)^{\frac{1}{p - 1}}, \quad |y(i)| =  \left( \frac{ |v(i) + \ell(i)| d(y)}{p} \right)^{\frac{1}{p - 1}} ~.$$
In particular we immediately see that \eqref{eq:stayclosetobdy} implies \eqref{eq:coordwisebd} by the triangle inequality (also $d(y) \leq 1$ and $p \geq 1$) as follows:
\begin{eqnarray*}
|y(i)| = \left( \frac{ |v(i) + \ell(i)| d(y)}{p} \right)^{\frac{1}{p - 1}} & \leq & \left( \frac{ 2\max(|v(i)|, |\ell(i)|) d(y)}{p} \right)^{\frac{1}{p - 1}} \\
& \leq & \max\left(\left(\frac{2 d(y)}{d(x)}\right)^{\frac{1}{p-1}} |x(i)|, |2 \ell(i)|^{\frac1{p-1}}\right) \\
& \leq & 8^{\frac{1}{p-1}} |x(i)| + |2 \ell(i)|^{\frac{1}{p-1}} ~.
\end{eqnarray*}
We now move to the proof of \eqref{eq:stayclosetobdy}. We first note that \eqref{eq:stayclosetobdy} is trivially true for $d(x) \geq 1/4$ and thus without loss of generality one can assume $\|x\|_p^p \geq 3/4$. Crucially we now consider two cases, depending on whether the non-zero coordinate of $\ell$ is a ``light'' or ``heavy'' coordinate in $x$. Let us assume $\ell(1) \neq 0$. If $x(1)\leq (1/2)^{1/p}$ (i.e., ``light'') then $\sum_{i \geq 2} |x(i)|^p \geq 1/4$ and thus
$$\|y\|_p^p \geq \sum_{i \geq 2} |y(i)|^p = \sum_{i\geq 2} |x(i)|^p \left(\frac{d(y)}{d(x)}\right)^{\frac{p}{p-1}} \geq \frac1{4} \left(\frac{d(y)}{d(x)}\right)^{\frac{p}{p-1}} ~,$$
which implies $d(y) \leq 4 d(x)$ (since $\|y\|_p \leq 1$). On the other hand if $x(1) \geq (1/2)^{1/p}$ (i.e., ``heavy'') then one has
$$|v(1)| =  \frac{p}{d(x)} |x(1)|^{p-1} \geq 2 ~,$$
and thus $|v(1) + \ell(1)| \geq \frac12 |v(1)|$ (since $|\ell(1)| \leq 1$) which implies
$$1 \geq |y(1)| \geq |x(1)| \left(\frac{d(y)}{2 d(x)}\right)^{\frac{1}{p-1}} \geq \left(\frac{d(y)}{4 d(x)}\right)^{\frac{1}{p-1}} ~.$$
\end{proof}

Finally we have:

\begin{proof}[of Lemma \ref{lem:locnormellp}]
Using successively Lemma \ref{lem:hessellp}, \eqref{eq:stayclosetobdy}, \eqref{eq:coordwisebd}, and the fact that $p \in [1,2]$, one has
\begin{eqnarray*}
\|\ell\|_{y,*}^2 \leq \frac{d(y)}{p(p-1)} \sum_{i=1}^n |y(i)|^{2-p} \ell(i)^2 & \leq & \frac{4 d(x)}{p(p-1)} \sum_{i=1}^n |y(i)|^{2-p} \ell(i)^2 \\ 
& \leq & \frac{4 d(x)}{p(p-1)} \sum_{i=1}^n (2^{\frac{3}{p-1}} |x(i)| + |2 \ell(i)|^{\frac{1}{p-1}})^{2-p} \ell(i)^2 \\
& \leq & \frac{2^{\frac{3}{p-1}} d(x)}{p(p-1)} \sum_{i=1}^n (|x(i)|^{2-p} + |\ell(i)|^{\frac{2-p}{p-1}}) \ell(i)^2 ~.
\end{eqnarray*}
\end{proof}

\subsection{Proof of Theorem \ref{th:LBellp}} \label{sec:proofLBellp}

For sake of clarity we write $\cK = \{(x,y) \in \R \times \R^n : |x|^p + \|y\|_p^p \leq 1\}$ and the losses as $\ell_t = (w_t, z_t) \in \R \times \R^n$. 
Let $\epsilon >0$ to be such that $\epsilon^q = C/\sqrt{T}$ for some small enough universal constant $C\in (0,1)$ (in particular since $T>n^2$ one has $\epsilon^q n < 1$). We now define i.i.d. Gaussian losses as follows. For $\xi \in \{-1,1\}^n$ let $\ell_t^{\xi}=(w_t, z_t^{\xi})$ where $w_t \sim \cN(-1,1)$ and $z_t^{\xi} \sim \cN(\epsilon \xi, \frac{1}{n^{2/q}} I_n)$. We show that
$$\E_{\xi} \E_{\ell_t^{\xi}} R_T =  \Omega(n \sqrt{T}) ~,$$
which clearly concludes the proof (notice since $T>n^2$ one has $\E \|\ell_t\|_q^q = O(1)$ and thus by rescaling by a constant one can also get \eqref{eq:gaussianbound}).
\newline

The key idea of the proof is to distinguish between ``exploration rounds'' and ``exploitation rounds'', depending on whether the played action $(x_t, y_t) \in \cK$ satisfies $x_t \leq 1/4$ or $x_t \geq 1/4$. Exploration rounds suffer constant regret because the optimal action $(x^*, y^*)$ has $x^*$ close to $1$. On the other hand exploitation rounds give little information about $\xi$ because of the constant variance induced by the $x$ component. Furthermore low-regret exploitation rounds should actually have the $x$ component close to $1$ which means that even less information about $\xi$ is gathered. We make this tradeoff more precise below, but first in Lemma \ref{step1} we formalize the fact that identifying $\xi$ matters for low-regret and in Lemma \ref{step2} we formalize the previous sentence.

Let us define $(\bar{x}, \bar{y}) = \frac{1}{T}\sum_{t = 1}^T \E[(x_t, y_t)]$ and $(x^*,y^*) = \argmin_{(x,y) \in \cK} x + \epsilon \xi \cdot y$. In particular one has
\begin{equation} \label{eq:lastday}
\E_{\ell_t^{\xi}} \frac{R_T}{T} \geq - (\bar{x} - x) + \epsilon \xi \cdot (\bar{y}-y^*) ~.
\end{equation}
We say a coordinate $i \in [n]$ is wrong if $\bar{y}(i) \xi(i) \geq 0$. 

\begin{lemma} \label{step1} 
Let $s$ be the number of wrong coordinates, then $\E_{\ell_t^{\xi}} R_T \geq \epsilon^q sT/ 4$. 
\end{lemma}

\begin{proof}
Let us assume that the first $s$ coordinates are wrong. A straightforward calculation shows that $-x^* + \epsilon \xi \cdot y^* =  - (1 + \epsilon^q n)^{1/q} $, and thus by \eqref{eq:lastday} it suffices to show that
$$-\bar{x} + \epsilon \sum_{i = s+1}^n \bar{y}(i) \xi(i) \geq \epsilon^q s / 4 - (1 + \epsilon^q n)^{1/q} ~.$$
 Since $\|(\bar{x}, \bar{y}({s + 1}), \cdots, \bar{y}(n))\|_p \leq 1$, by Holder's inequality we know that 
 $$\bar{x} - \epsilon\sum_{i = s + 1}^n \bar{y}(i) \xi(i) \leq (1 + \epsilon^q (n - s))^{1/q} ~.$$
This concludes the proof since $(1 + \epsilon^q (n - s))^{1/q} \leq (1 + \epsilon^q n)^{1/q} - \frac{1}{2q} \epsilon^q  s$.
\end{proof}

\begin{lemma} \label{step2} 
$\bar{x} \leq 1- 4 \epsilon^q n \Rightarrow \E_{\ell_t^{\xi}} R_T \geq \epsilon^q n T$. 
\end{lemma}

\begin{proof}
It suffices to show that $-\bar{x} + \epsilon \xi \cdot \bar{y} \geq \epsilon^q n - (1 + \epsilon^q n)^{1/q}$ (see beginning of previous proof). Observe that
$$-\bar{x} + \epsilon \xi \cdot \bar{y} \geq -|\bar{x}| - \epsilon \|\xi\|_q \|\bar{y}\|_p \geq - |\bar{x}| - (1-|\bar{x}|^p)^{1/p} \epsilon n^{1/q} ~.$$
Observe that $x \mapsto x + (1 - x^p)^{1/p} \epsilon n^{1/q}$ is a nondecreasing function for $x \in [0, 1 - \epsilon^q n]$ since
$$\frac{1}{p} \epsilon n^{1/q} (1-(1- \epsilon^q n)^p)^{1/p-1} \leq \epsilon n^{1/q} (\epsilon^q n)^{1/p - 1} = 1 ~.$$
Therefore we have
$$-\bar{x} + \epsilon \xi \cdot \bar{y}  \geq - (1 - 4 \epsilon^q n) - (1-(1 -4 \epsilon^q n)^p)^{1/p} \epsilon n^{1/q} ~, $$
and thus the proof is concluded by $1 + (1-(1 -4 \epsilon^q n)^p)^{1/p} (\epsilon^q n)^{1/q} \leq (1 + \epsilon^q n)^{1/q} + 3 \epsilon^q n$.
\end{proof}

Observe now that the observed feedback at round $t$ is exactly
$$f_t^{\xi} := x_t w_t + y_t \cdot z_t^{\xi} \sim \cN(x_t + \epsilon y_t \cdot \xi, \sigma_t^2), \; \text{where} \; \sigma_t^2 = x_t^2 + \|y_t\|_2^2 / n^{2/q} ~.$$
Denote $\cL_{\xi}$ for the law of the observed feedback up to time $T$, i.e., the law of $(f_1^{\xi}, \hdots, f_T^{\xi})$. Standard calculations show that for $\xi$ and $\xi'$ differing only in coordinate $i \in [n]$ one has
$$\mathrm{TV}(\cL(\xi), \cL(\xi')) \leq \sqrt{\sum_{t=1}^T \E_{\ell_t^{\xi}} \frac{\epsilon^2 y_t(i)^2}{\sigma_t^2}} ~.$$
Another standard calculation show that the above inequality implies
$$\E_{\xi, \ell_t^{\xi}} \frac{1}{T}\sum_{t=1}^T \sum_{i=1}^n \ds1\{y_t(i) \xi(i) < 0\} \geq \frac{n}{2} - \sqrt{n \sum_{t=1}^T \E_{\xi, \ell_t^{\xi}} \frac{\epsilon^2 \|y_t\|_2^2}{\sigma_t^2}} ~.$$
Note that the left hand side in the above inequality is exactly the average (over time) number of wrongly guessed coordinates for $\xi$, which we know controls the regret thanks to Lemma \ref{step1}. In particular it only remains to show that 
\begin{equation} \label{eq:X}
\sum_{t=1}^T \E_{\xi, \ell_t^{\xi}} \frac{\epsilon^2 \|y_t\|_2^2}{\sigma_t^2} \leq c n ~,
\end{equation}
for some universal constant $c<1/2$.
\newline

Note that one always has $\sigma_t^2 \geq \|y_t\|_2^2/ n^{2/q}$ and furthermore $x_t \geq 1/4 \Rightarrow \sigma_t^2 \geq 1/2^4$. Recall also that $\|y_t\|_2 \leq n^{1/2-1/p} \|y_t\|_p \leq n^{1/2-1/p} (1-|x_t|^p)^{1/p}$. Thus 
\begin{equation} \label{eq:Y}
\E \sum_{t=1}^T \frac{\epsilon^2 \|y_t\|_2^2}{\sigma_t^2} \leq n^{2/q} \epsilon^2 \E \ \sum_{t=1}^T \ds1\{x_t \leq 1/4\} + 2^4 \epsilon^2 n^{1-2/p} \sum_{t : x_t \geq 1/4} \E (1-|x_t|^p)^{2/p} ~.
\end{equation}
Observe that one clearly has $\E R_T = \Omega(\E \sum_{t=1}^T \ds1\{x_t \leq 1/4\} )$ and thus without loss of generality we can assume $\E \sum_{t=1}^T \ds1\{x_t \leq 1/4\}  = O(n \sqrt{T})$, which means that the first term on the right hand side in \eqref{eq:Y} is smaller than $n^{1+2/q} \epsilon^2 \sqrt{T} = C^{2/q} n^{1+2/q} T^{1/2-1/q}$. This is smaller than $n$ for $T \geq n^{\frac{2}{1-q/2}}$ and $C$ small enough. For the second term we use that
\begin{eqnarray*}
\sum_{t : x_t \geq 1/4} \E (1-|x_t|^p)^{2/p}  & \leq & p^2 \sum_{t = 1}^T \E (1-|x_t|)^{2/p} \\
& \leq & p^2 T\left( \E \left(1- \frac1{T} \sum_{t=1}^T |x_t|\right) \right)^{2/p} ~,
\end{eqnarray*}
and because of Lemma \ref{step2} one can assume $\frac1{T} \E[\sum_{t=1}^T |x_t| ]\geq 1 - 4 \epsilon^q n$ which means that the second term in \eqref{eq:Y} is smaller than $\epsilon^2 n^{1-2/p} T (\epsilon^q n)^{2/p}= \epsilon^{2 q} n T = C^2 n$. This concludes the proof of \eqref{eq:X}, and thus also concludes the proof of Theorem \ref{th:LBellp}.

\subsection{Proof of Theorem \ref{th:starved}}
We only give a brief proof sketch. The starved multi-armed bandit lower bound is standard and can be written succintly as follows. Consider random losses, where say action $1$'s loss is a Bernoulli of parameter $1/2$ plus or minus $\epsilon$, action $2$ is a Bernoulli of parameter $1/2$, and all the other actions always give a loss of $1$. Denote by $E$ the expected number of exploration rounds, i.e. rounds where the player plays from $\mu$. It is a standard calculation that if $E/n \leq c / \epsilon^2$ for some sufficiently small constant $c$, then the regret is at least $\epsilon T$. On the other hand the regret is always larger than $\frac{n-2}{n} E /2$. Thus by setting $\epsilon^2 = c n / E$ we have a regret lower bounded by (up to constant), with $a$ such that $a= (1-a) \frac{1}{2}$ (i.e., $a=1/3$):
$$\max\left(E, \left(\frac{n}{E}\right)^{1/2} T \right) \geq n^a T^{1-a} ~.$$
Essentially the same argument applies to the $\ell_1^n$ ball, we omit the details. We now turn to the case of $\ell_p^n$ balls with $p>2$.

We see from \eqref{eq:X} (observe that in the starved setting the sum over all $t \in [T]$ in this equation is replaced by the sum over rounds $t$ where one plays from $\mu$) that if $n^{2/q} \epsilon^2 E \leq c n$ for some sufficiently small constant $c$, then the regret is at least $\epsilon^q n T$ (per Lemma \ref{step1}). Moreover the regret is also always larger than $E$. Thus by setting $\epsilon^2 = c n^{1-2/q} / E$ (i.e., $\epsilon^q n = C (n / E)^{q/2}$) we have a regret lower bounded by (up to a constant), with $a$ such that $a = (1-a) q/2$,
$$\max\left(E, \left(\frac{n}{E}\right)^{q/2} T \right) \geq n^a T^{1-a} ~,$$
which concludes the proof.

\bibliographystyle{plainnat}
\bibliography{newbib}
\end{document}